%% file: content.tex
\documentclass[sigconf]{acmart}

\newtheorem{assumption}{Assumption}
\usepackage{subcaption}
\usepackage{multicol}
\usepackage{multirow}
\usepackage{adjustbox}
\usepackage{rotating}

\definecolor{firstcolor}{HTML}{0000ff}
\definecolor{secondcolor}{HTML}{cb0707}
\newcommand\colorup[1]{\textcolor{firstcolor}{\textbf{#1}}}
\newcommand\colordown[1]{\textcolor{secondcolor}{\textbf{#1}}}

\newcommand\revised[1]{\textcolor{blue}{}}
\AtBeginDocument{%
  }

\setcopyright{acmlicensed}
\copyrightyear{2026}
\acmYear{2026}
\acmDOI{XXXXXXX.XXXXXXX}
\acmConference[CIKM '26]{Conference on Knowledge and Information Management}{November 07--11,
  2026}{Rome, Italy}
\acmISBN{978-1-4503-XXXX-X/2018/06}

\begin{document}

\title{Local Message-Passing for Discrete Graph Generation}

\author{Jay Revolinsky}
\email{revolins@msu.edu}
\orcid{1234-5678-9012}
\affiliation{%
  \institution{Michigan State University}
  \city{East Lansing}
  \state{Michigan}
  \country{USA}
}

\author{Harry Shomer}
\affiliation{%
  \institution{University of Texas - Arlington}
  \city{Arlington}
  \state{Texas}
  \country{USA}
}

\author{Jiliang Tang}
\affiliation{%
  \institution{Michigan State University}
  \city{East Lansing}
  \state{Michigan}
  \country{USA}
}

\renewcommand{\shortauthors}{Revolinsky et al.}

\begin{abstract}
    Discrete graph generation has emerged as a powerful paradigm for modeling graph-structured data, yet state-of-the-art models often rely on Graph Transformers or higher-order architectures. We revisit this design assumption by introducing GenGNN, a modular message-passing backbone for graph generation. GenGNN enables powerful generation by persisting edge fields through latent refinement of coupled node–edge–graph states, all without requiring global attention. Diffusion models integrating GenGNN achieve over 90\% validity on standard benchmark datasets, performing within margins of Graph Transformer backbones and achieving 2–5x faster inference. Systematic ablations isolate how GenGNN is resilient to oversmoothing during generative (de)noising, indicating each GenGNN component is necessary for downstream generation quality. Finally, representation-space analysis suggests GenGNN learns functionally-similar representations to more theoretically-expressive architectures; even at deeper layers. As such, GenGNN uplifts local message-passing to challenge prevailing assumptions that performant discrete graph generation requires global attention or higher-order representations. Source Code: \href{https://github.com/revolins/GenGNN_CIKM}{Available Here}
\end{abstract}



\received{23 May 2026}
\received[revised]{13 August 2026}
\received[accepted]{20 August 2026}

\maketitle

\input{sections/intro}
\input{sections/background}

\input{sections/gengnn}
\input{sections/results}

\input{sections/conclusion}

\clearpage
\bibliographystyle{ACM-Reference-Format}
\bibliography{references}


\end{document}

%% file: sections/intro.tex
\section{Introduction}

Graph generation is a rapidly growing subfield of graph representation learning, with applications to drug discovery~\cite{minkai2023geoldm} and code modeling~\cite{brockschmidt2018generative}. Recent graph generative models (GGMs) based on discrete diffusion achieve strong downstream performance by iteratively denoising discrete node and edge states~\cite{haefeli2022diffusion, vignac2022digress}. More recent work further integrates continuous-time flow modeling into discrete graph generation~\cite{siraudin2024cometh, xu2024discrete}, resulting in strong performance on standard benchmark datasets~\cite{qin2024defog}. However, this performance comes with a substantial inference cost, as discrete graph generation typically requires repeated denoising steps over dense node--edge representations which are quadratic in complexity~\cite{vignac2022digress, qin2024defog}. 

\begin{figure}[t!]
    \centering
    \begin{subfigure}[t]{0.22\textwidth}
        \centering
        \includegraphics[height=1in]{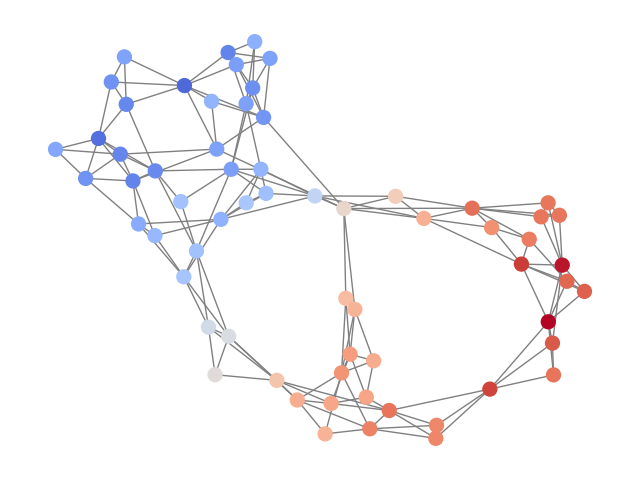}
        \caption{Simple GNN}
        \label{fig:degen_planar}
    \end{subfigure}
    \begin{subfigure}[t]{0.22\textwidth}
        \centering
        \includegraphics[height=1in]{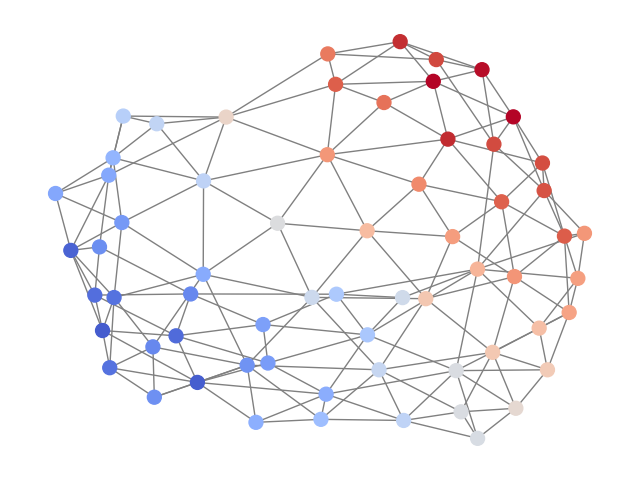}
        \caption{Graph Transformer}
        \label{fig:good_planar}
    \end{subfigure}%
    \caption{Planar graphs generated via DeFoG with either a simple GNN or Graph Transformer backbone. The simple GNN fails to reproduce finer substructures in its output.}
    \label{fig:toy_planar}
\end{figure}

Existing work has primarily addressed this cost by improving the generative process itself. This includes: reducing the number of denoising steps~\cite{geng2025improved}, improving transition modeling~\cite{song2020denoising}, or sparsifying reconstruction~\cite{qin2023sparse}. In contrast, a complementary bottleneck remains underexplored: {\bf the complexity of the denoising backbone used at each generative step}. Existing GGMs~\cite{vignac2022digress, qin2024defog} rely on computationally expensive~\cite{muller2024attending} and specialized denoising backbones including variants of the transformer architecture~\cite{vaswani2017attention, rampavsek2022recipe, muller2024attending} and maximally expressive graph networks~\cite{maron2019provably, morris2019weisfeiler}. Therefore, even when the number of denoising steps is reduced, the high complexity of the denoising model reduces the potential speedup. This is visualized in Figure~\ref{fig:sampling_diagram}, where each sampling step incurs a function call ($f_\theta$) to the denoising model; the cost of which scales with model complexity.

As opposed to standard GNNs which rely on local message passing~\cite{kipf2016semi},  expressive GGM denoising models (e.g., graph transformers) are non-local in nature, allowing for the to encode more complicated graph structures. The use of these models follow a natural two-part intuition. \textbf{First}, standard graph neural networks (GNNs) are limited in their ability to distinguish certain graph substructures~\cite{xu2018powerful}. \textbf{Second}, deeper GNNs can suffer from over-smoothing, particularly when many layers are required to propagate structural information across denoising timesteps~\cite{arnaiz2025oversmoothing}. Together, these limitations suggest that local GNNs may fail to accurately sample graph structure during iterative generation. Figure~\ref{fig:toy_planar} illustrates this intuition with a simple example: a standard GNN backbone fails to fully recover a Planar graph, while a Graph Transformer successfully samples the same input structure.

Evidence from node- and graph-classification tasks suggests a more nuanced story. Recent work shows carefully modified GNNs achieve performance comparable to Graph Transformers (GTs) on several graph learning tasks~\cite{liao2024greener, luo2024classic, luo2025can}. This is important because local GNNs are typically more efficient than architectures based on tuple representations~\cite{maron2019provably, morris2019weisfeiler} or global attention~\cite{rampavsek2022recipe, muller2024attending}, enabling possible inference speedups~\cite{luo2025can}. Unlike classification, graph generation poses a distinct challenge: a denoising backbone must repeatedly predict transitions over nodes, edges, and graphs. Thus, the success of efficient GNNs in classification does not immediately imply local message-passing can serve a viable graph-generation backbone. This motivates a central question: \textit{can efficient local message-passing serve as an effective denoising backbone for discrete graph generation?}

To challenge the use of highly expressive denoising models, we introduce \textsc{GenGNN}, an efficient generation-oriented message-passing backbone for discrete graph generation. Rather than treating local message-passing as a generic replacement for Graph Transformers, \textsc{GenGNN} intentionally restricts the denoising model to local updates for testing if limitations of standard GNNs can be mitigated through targeted design. \textsc{GenGNN} jointly evolves node, edge, and graph-level latent states through a modular message-passing scheme. This scheme combines edge-conditioned updates, residualized feed-forward transformations, normalization, and graph-level conditioning into a recurrent loop. This design targets the central difficulty of graph generation for message-passing: maintaining coherent structural information across repeated denoising steps without relying on global attention or higher-order tensor representations.

Empirically, \textsc{GenGNN} can achieve competitive performance with state-of-the-art GGM encoders while also {\bf improving efficiency by} $\mathbf{1.7}$--$\mathbf{5\times}$ on standard graph-generation benchmarks, including Tree, Planar, SBM, ZINC, QM9, GuacaMol, and MOSES~\cite{bergmeister2023efficient, martinkus2022spectre, sterling2015zinc, wu2018moleculenet, polykovskiy2020molecular, brown2019guacamol}. Our theoretical analysis further indicates that residual connections help ground latent \textsc{GenGNN} representations across denoising timesteps, mitigating over-smoothing even when the input graph is highly noised. Ablation studies show that each \textsc{GenGNN} component contributes to downstream generation quality, while depth-scaling analysis demonstrates that \textsc{GenGNN} remains resilient to over-smoothing and learns latent representations~\cite{limbeck2024metric} comparable to more expressive architectures. These findings suggest that performant discrete graph generation does not necessarily require expensive non-local denoising backbones; instead, carefully structured local message-passing can provide a practical efficiency--performance tradeoff.
The contribution of our work is as follows:
\begin{itemize}
  \item We introduce GenGNN, a local message-passing framework for strong graph generation with 1.7-5x relative inference speedup.
  \item We detail a theoretical underpinning on how GenGNN overcomes oversmoothing for effective discrete graph diffusion.
  \item We conduct detailed ablation and layer-depth analysis to demonstrate GenGNN's connection between downstream performance and latent representations.
\end{itemize}

%% file: sections/background.tex
\section{Background}

\subsection{Preliminary}

Within this work, we represent a graph $G = \{\mathbf{X}, \mathbf{E}\}$ by node features
$\mathbf{X} \in \mathbb{R}^{N \times d_x}$,
edge features
$\mathbf{E} \in \mathbb{R}^{N \times N \times d_e}$,
and a global conditioning vector with a timestep embedding as
$\mathbf{y} \in \mathbb{R}^{d_y}$. The adjacency matrix for a target graph is illustrated by ${A}^{N \times N}$, where a point of interest is shown by $a_{ij}$.

We denote the hidden representations at layer $\ell$ by
$\mathbf{X}^{(\ell)}$, $\mathbf{E}^{(\ell)}$, and $\mathbf{y}^{(\ell)}$. A continuous flow-matching rate matrix is represented by 
${R}^{N \times N }$. Discrete probabilities computed across an adjacency matrix are shown as ${Q}^{N \times N}$ or $q(.)$. The transition matrix is  detailed with $M^{N \times N}$, which is distinct from a message ($\mathcal{M}_{\theta}$) or $\mathbf{m}_{i}$.

\begin{figure*}[t!h]
  \begin{center}
    \centerline{\includegraphics[width=0.9\textwidth]{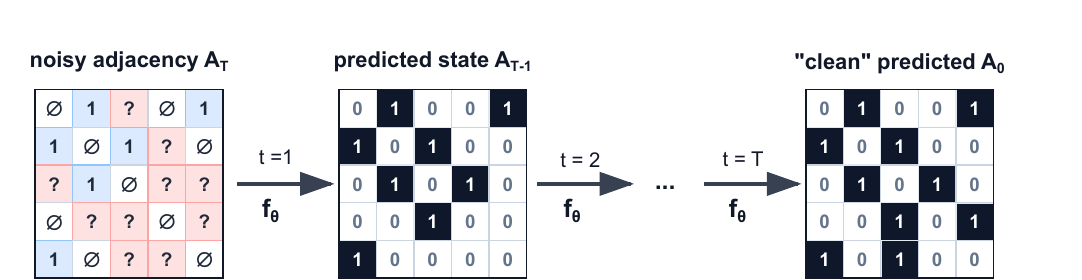}}
    \caption{
       The discrete denoising process for sampling a "clean" output adjacency matrix, based on a noisy categorical input distribution. This work is centered around $f_{\theta}$'s complexity, with the intent to balance local message-passing expressivity and sampling efficiency.
    }
    \label{fig:sampling_diagram}
  \end{center}
\end{figure*}


\subsection{Discrete Graph Diffusion}

Discrete-state diffusion transitions between (de)noised states via marginal probabilities on categorical distributions. Contrary to continuous diffusion~\cite{sohl2015deep}, discrete diffusion adding noise to form $G_t = (X_t, E_t)$, where $t$ is the $t$-th diffusion timestep. This amounts to sampling each node and edge type from the categorical distribution~\cite{vignac2022digress}:
\begin{align}
\begin{split}
q(M_t \mid M_{t-1}) &= \bigl(X_{t-1} M_t^{X}, \; E_{t-1} M_t^{E}\bigr),\\
q(M_t \mid M) &= \bigl(X \bar{M}_t^{X}, \; E \bar{M}_t^{E}\bigr),
\label{eq:forward_noising}
\end{split}
\end{align}
where $\bar{M}$ is the transition matrix for calculating the diffusion process.
In practice, Eq.~\eqref{eq:forward_noising} enables a predictive encoder to learn a discrete diffusion process for strong graph generation performance. Recent work integrated continuous flow-modeling into the discrete-space, enabling multimodal learning on discrete sequences~\cite{campbell2024generative}. In order for flow-modeling to operate directly on graph structure, DeFoG~\cite{qin2024defog} computes transition dynamics from a rate matrix as follows:  
\begin{equation}
R_t^{(n)}\!\left(x_t^{(n)}, x_{t+\Delta t}^{(n)}\right)
= \\
\mathbb{E}_{p_{1 \mid t}^{(n)}\!\left(x_1^{(n)} \mid G_t\right)}
\Bigl[
R_t^{(n)}\!\left(x_t^{(n)}, x_{t+\Delta t}^{(n)} \mid x_1^{(n)}\right)
\Bigr],
\label{eq:node_rate}
\end{equation}
and analogously for the edge transitions, $\tilde{p}_{t+\Delta t \mid t}^{(ij)}\!\left(e_{t+\Delta t}^{(ij)} \mid G_t\right)$ where $\tilde{p}(\cdot)$ represent the probability of an edge existing between nodes $i$ and $j$ at time $t + \Delta t$. The rate matrix computation allows a decoupling of discrete denoising, demonstrating state-of-the-art performance with transformer encoders. Although specialized message-passing demonstrate impressive capability for protein sequence modeling~\cite{dauparas2022robust}, it remains untested for structured learning in the graph generation task.

\subsection{Message-Passing for Graph Generation}

Local message-passing operates exclusively on nodes within a given adjacency matrix~\cite{kipf2016semi}, fundamentally constraining graph generation to node-transition states. To bypass this requirement, DiGress~\cite{vignac2022digress} utilized a graph transformer with node-to-edge self-attention and graph-conditioning via FiLM~\cite{perez2018film}. 

So, \textit{effective} local message-passing~\cite{gilmer2017neural} as a backbone for graph generation remains a unique challenge. An unaltered denoising backbone must learn transitions for extrapolating nodes ($X$) to edge ($E$) and graph ($y$) transitions between subsequent layers ($l$):  

\begin{equation}
X^{\ell}
\rightarrow
E^\ell, y^\ell
\rightarrow
X^{\ell+1}
\rightarrow
E^{\ell+1}, y^{\ell+1}
\label{eq:layer_loop}
\end{equation}

As shown in Figure~\ref{fig:degen_planar}, standard message-passing trained via this sequential process remains insufficient for effective graph generation~\cite{bergmeister2023efficient, xu2024discrete}. 

Given this reduced performance, graph diffusion models typically rely on encoders capturing long-range interactions or fine-grained structural information. This includes graph transformer~\cite{vignac2022digress} or PPGN backbones~\cite{haefeli2022diffusion}. These methods are typically endowed with additional feature encodings to enhance the discriminative power of GGMs. SPECTRE inputs eigenvectors as additional features, preventing mode collapse in adversarial generation~\cite{martinkus2022spectre}. COMETH and DeFoG append RRWP encodings~\cite{ma2023graph} within their frameworks, approximating $p$-cycles counts for fine-grained generation~\cite{siraudin2024cometh, qin2024defog}. Separate investigations into graph diffusion models demonstrate $>$2-WL expressive GNNs enhance graph substructure generation~\cite{wang2025graph}. Therefore, current literature suggests maximal-expressivity is all but required for performant graph generation.


\paragraph{Oversmoothing in GNNs} A common concern with vanilla GNNs is their over-smoothing of node signals when propagating distant information~\cite{alon2020bottleneck}. Whether this phenomenon translates to discrete graph diffusion remains crucially untested, making it infeasible to determine if GNNs are capable diffusion backbones. Indeed, foundational work~\cite{li2019deepgcns} indicates added residual connections allows node-classification with deep graph convolutional networks (GCNs)~\cite{kipf2016semi}, especially for learning on long-range temporal information. Recent theoretical work demonstrates that including normalization and residual connections `anchor' initial feature signals within GNNs, enabling retention of relevant signals in deep layers~\cite{scholkemper2024residual}. We pull inspiration from these foundational works in the next section.

\subsection{Evaluation}

Foundational work testing GGM evaluation metrics~\cite{thompson2022evaluation} revealed that maximum mean discrepancy (MMD)~\cite{gretton2012kernel} effectively measures whether GGMs successfully sample local structures (i.e. orbits, clustering)~\cite{hovcevar2014combinatorial}. However, MMD requires tuning to measure all potential modes within sampling distributions. To remedy this, \cite{limbeck2024metric} propose MagDiff~\cite{leinster2013magnitude} which solves the integral between the difference in area between sampled and reference embeddings: \begin{equation}
\label{eq:magdiff}
\mathrm{MagDiff}
\;:=\;
\int_{t_0}^{t_{\mathrm{cut}}}
\big(
\mathrm{Mag}_X(t) - \mathrm{Mag}_Y(t)
\big)\, \mathrm{d}t,
\end{equation}

As such, MagDiff gauges the diversity of downstream graphs samples, enabling determination of whether a GGM learns to re-create it's training distribution without needing to tune a metric.

\paragraph{Comparison with Vanilla GNNs} For initial verification, we compare GenGNN to vanilla GNNs (i.e. GNNs with no enhancements discussed in Section~\ref{sec:GenGNN} ). This includes GCN~\cite{kipf2016semi}, GIN~\cite{xu2018powerful}, and MPNN~\cite{gilmer2017neural}. The comparison is done on the QM9 and ZINC datasets for the DeFoG generative model. We observe a sharp decrease in performance, showing that vanilla GNNs are not enough, and GenGNN's added components are necessary to boost the performance of standard GNNs.

\begin{table}[h!]
\centering
\small
\begin{adjustbox}{width=0.95\columnwidth}
\begin{tabular}{l|cc|cc}
\toprule
\multirow{2}{*}{\textbf{Model}} & \multicolumn{2}{c|}{\textbf{QM9}} & \multicolumn{2}{c}{\textbf{ZINC}} \\
\cmidrule(lr){2-3} \cmidrule(lr){4-5}
& Val $\uparrow$ & Uniq $\uparrow$ & Val $\uparrow$ & Uniq $\uparrow$\\
\midrule

MPNN/GNN~\cite{gilmer2017neural}
& 20.51 $\pm$ 0.16 & 66.2 $\pm$ 3.45
& 31.1 $\pm$ 0.3 & 99.98 $\pm$ 0.02 \\

GCN~\cite{kipf2016semi}
& 18.19 $\pm$ 0.38 & 77.11 $\pm$ 0.54
& 2.15 $\pm$ 0.1 & 83.05 $\pm$ 0.8  \\

GIN~\cite{xu2018powerful}
& 39.44 $\pm$ 0.32 & 88.98 $\pm$ 0.57
& 82.84 $\pm$ 0.2 & 99.61 $\pm$ 0.08  \\

\midrule 
GenGNN & 99.26 $\pm$ 0.06 & 96.38 $\pm$ 0.01 & 95.55 $\pm$ 0.79 & 99.95 $\pm$ 0.14 \\

\bottomrule
\end{tabular}
\end{adjustbox}
\caption{Comparison of basic GNNs and GenGNN for DeFoG on QM9 and ZINC.}
\label{tab:compare_gnns}
\end{table}

%% file: sections/gengnn.tex
\section{GenGNN: A Simple Message Passing Model for Graph Generation} \label{sec:GenGNN}



Current encoder backbones for GGMs~\cite{vignac2022digress, qin2024defog} tend to use powerful but expensive model designs which hamper efficiency. On the other hand, recent work shows simpler models achieve first-class performance in traditional node~\cite{luo2024classic} and graph classification~\cite{luo2025can, luo2024classic}. We are therefore interested in investigating whether {\it similar graph generation performance can be achieved via simpler GNN models}? 

We propose a simple GNN design to test whether best practice from other tasks~\cite{luo2025can, luo2024classic} transfer to graph generation, all while maintaining strong efficiency for reducing backbone model complexity during generative sampling. Therefore, reducing resource strain during sampling steps (Figure~\ref{fig:sampling_diagram}) and ensuring the best-possible model performance. Our framework, Generative Graph Neural Network (GenGNN), integrates message-passing neural networks into discrete diffusion models as effective generative backbones. This effectiveness is possible via modular components integrated after local message-passing layers.

\begin{figure*}[t!h]
  \begin{center}
    \centerline{\includegraphics[width=0.9\textwidth]{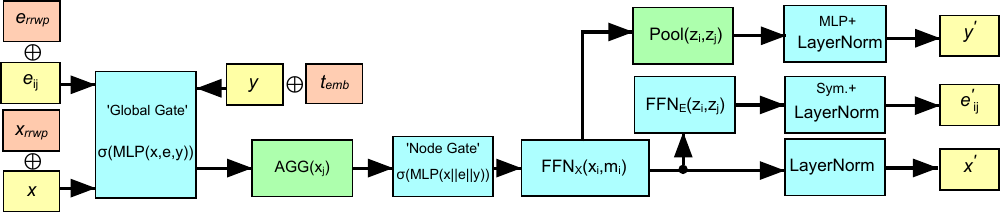}}
    \caption{
       The per-layer GenGNN framework, composed of modular components, in order: Node ($x$), Edge ($e_{ij}$), and Graph (y) Features w/ RRWP or time embeddings (shown in Orange). Residual channels flow between initial features and output features. \textit{Note: } Blue and orange modules are ablatable. Yellow and green modules are always enabled during experimentation. We denote the output graph as $y'$, $e'$, and $x'$.
    }
    \label{fig:gengnn_diagram}
  \end{center}
\end{figure*}

\subsection{Model Design}


GenGNN is designed as a coupled node--edge--global message-passing operator for graph generation. Rather than treating node updates, edge updates, global readouts, feed-forward transformations, residual connections, and normalization as independent modules, GenGNN composes them into a single denoising transition that evolves node states $\mathbf{X}^{(\ell)}$, edge states $\mathbf{E}^{(\ell)}$, and graph-level states $\mathbf{y}^{(\ell)}$ across layers. This is important in graph generation because denoising backbones simultaneously recover local node attributes, pairwise edge structure, and global graph conditioning. We therefore view each GenGNN layer as the transition
\begin{equation}
\left(
\mathbf{X}^{(\ell)},
\mathbf{E}^{(\ell)},
\mathbf{y}^{(\ell)}
\right)
\xrightarrow{\;\mathcal{f}_{\theta}^{(\ell)}\;}
\left(
\mathbf{X}^{(\ell+1)},
\mathbf{E}^{(\ell+1)},
\mathbf{y}^{(\ell+1)}
\right),
\label{eq:gengnn_transition}
\end{equation}
where $\mathcal{f}_{\theta}^{(\ell)}$ couples local message passing with edge-state persistence and graph-level feedback.

\paragraph{Input Feature Representations.} Following prior work~\cite{ma2023graph,qin2024defog}, GenGNN initializes edge states with relative random walk probabilities (RRWP).
Let $\mathbf{T}$ denote the random-walk transition matrix and let $K$ be the maximum walk length. We define
\begin{equation}
\mathbf{r}_{ij}
=
\left[
\mathbf{I},\,
\mathbf{T},\,
\mathbf{T}^{2},\,
\ldots,\,
\mathbf{T}^{K-1}
\right]_{ij},
\label{eq:rrwp}
\end{equation}
and initialize the edge representation as
\begin{equation}
\mathbf{e}_{ij}^{(0)}
=
\left[
\mathbf{e}_{ij}^{\mathrm{attr}}
\,\Vert\,
\mathbf{r}_{ij}
\right].
\label{eq:edge_init_rrwp}
\end{equation}
This provides pairwise structural context before denoising begins. Thus bypassing the first critical limitation of local message-passing. Subsequent GenGNN layers then determine how structure is preserved or transformed during generation.

\paragraph{Global-aware Gated Aggregation.} At layer $\ell$, GenGNN first computes an edge-conditioned gate for each valid node pair $(i,j)$. Let $\mathbf{B}_{ij}\in\{0,1\}$ denote the pair mask, where invalid nodes and self-pairs are excluded. The global gate is defined as
\begin{equation}
g_{ij}^{(\ell)}
=
\sigma
\left(
\mathbf{w}_{g}^{\top}
\operatorname{GELU}
\left(
\mathbf{W}_{i}\mathbf{x}_{i}^{(\ell)}
+
\mathbf{W}_{j}\mathbf{x}_{j}^{(\ell)}
+
\mathbf{W}_{e}\mathbf{e}_{ij}^{(\ell)}
+
\mathbf{W}_{y}\mathbf{y}^{(\ell)}
+
\mathbf{b}_{g}
\right)
\right),
\label{eq:edge_gate}
\end{equation}
and the effective pairwise message weight is
\begin{equation}
a_{ij}^{(\ell)}
=
\mathbf{B}_{ij} g_{ij}^{(\ell)}.
\label{eq:effective_edge_weight}
\end{equation}
Thus, message passing is not performed over a fixed adjacency alone. Instead, each node-to-node signal is modulated by a combination of: source state, target state, edge state, and graph-level state.

The node message is then computed by normalized aggregation~\cite{bresson2017residual}:
\begin{equation}
\mathbf{m}_{i}^{(\ell)}
=
\frac{
\sum_{j}
a_{ij}^{(\ell)}
\mathbf{W}_{m}\mathbf{x}_{j}^{(\ell)}
}{
\sum_{j} a_{ij}^{(\ell)} + \epsilon
}.
\label{eq:gated_message}
\end{equation}
\paragraph{Node-Gating.} When node gating is enabled, GenGNN further modulates this aggregated message with the receiving node, the local edge context, and the global
graph state:
\begin{equation}
\bar{\mathbf{e}}_{i}^{(\ell)}
=
\frac{
\sum_{j}
\mathbf{B}_{ij}\mathbf{e}_{ij}^{(\ell)}
}{
\sum_{j}\mathbf{B}_{ij}+\epsilon
},
\label{eq:local_edge_pool}
\end{equation}
\begin{equation}
s_{i}^{(\ell)}
=
\sigma
\left(
\operatorname{MLP}_{v}
\left(
\left[
\mathbf{x}_{i}^{(\ell)}
\,\Vert\,
\bar{\mathbf{e}}_{i}^{(\ell)}
\,\Vert\,
\mathbf{y}^{(\ell)}
\right]
\right)
\right),
\label{eq:node_gate}
\end{equation}
\begin{equation}
\tilde{\mathbf{m}}_{i}^{(\ell)}
=
s_{i}^{(\ell)}
\odot
\mathbf{m}_{i}^{(\ell)}.
\label{eq:node_gated_message}
\end{equation}
If node gating is disabled, we set $s_i^{(\ell)}=\mathbf{1}$.

\paragraph{State Modulation via Feed-Forward.} Let $\mathbf{W}_{s}$ represent the weighted matrix obtained via gating. The node state is updated by combining the gated message ($\tilde{\mathbf{m}}_{i}^{(\ell)}$) into a feed-forward transformation:
\begin{equation}
\Delta \mathbf{x}_{i}^{(\ell)}
=
\operatorname{FFN}_{X}
\left(
\mathbf{x}_{i}^{(\ell)}
+
\mathbf{W}_{s}\mathbf{x}_{i}^{(\ell)}
+
\tilde{\mathbf{m}}_{i}^{(\ell)}
\right),
\label{eq:node_delta}
\end{equation}
\begin{equation}
\mathbf{x}_{i}^{(\ell+1)}
=
\operatorname{LayerNorm}
\left(
\mathbf{x}_{i}^{(\ell)}
+
\Delta \mathbf{x}_{i}^{(\ell)}
\right).
\label{eq:node_update}
\end{equation}

After the node state is refreshed, GenGNN updates each edge using the new endpoint states and the previous edge state:
\begin{equation}
\Delta \mathbf{e}_{ij}^{(\ell)}
=
\operatorname{FFN}_{E}
\left(
\left[
\mathbf{x}_{i}^{(\ell+1)}
\,\Vert\,
\mathbf{x}_{j}^{(\ell+1)}
\,\Vert\,
\mathbf{e}_{ij}^{(\ell)}
\right]
\right).
\label{eq:edge_delta}
\end{equation}
To preserve the undirected edge representation used by the denoising backbone, the edge update is symmetrized:
\begin{equation}
\mathbf{e}_{ij}^{(\ell+1)}
=
\mathbf{B}_{ij}
\operatorname{LayerNorm}
\left(
\mathbf{e}_{ij}^{(\ell)}
+
\frac{
\Delta \mathbf{e}_{ij}^{(\ell)}
+
\Delta \mathbf{e}_{ji}^{(\ell)}
}{2}
\right).
\label{eq:edge_update}
\end{equation}

Finally, the global graph state is updated from pooled node and edge states:
\begin{equation}
\bar{\mathbf{x}}^{(\ell+1)}
=
\frac{1}{|V|}
\sum_{i\in V}
\mathbf{x}_{i}^{(\ell+1)},
\qquad
\bar{\mathbf{e}}^{(\ell+1)}
=
\frac{
\sum_{i,j}\mathbf{B}_{ij}\mathbf{e}_{ij}^{(\ell+1)}
}{
\sum_{i,j}\mathbf{B}_{ij}+\epsilon
}.
\label{eq:global_pool}
\end{equation}
The graph-level state then receives a residual update:
\begin{equation}
\mathbf{y}^{(\ell+1)}
=
\operatorname{LayerNorm}
\left(
\mathbf{y}^{(\ell)}
+
\mathbf{W}_{y}
\left[
\bar{\mathbf{x}}^{(\ell+1)}
\,\Vert\,
\bar{\mathbf{e}}^{(\ell+1)}
\right]
\right).
\label{eq:global_update}
\end{equation}

\paragraph{Message-Persistence Loop.} Equations~\eqref{eq:edge_gate}--\eqref{eq:global_update} define the core framework components. Importantly, the graph-level state is not used just for a final readout. It participates in global gating and node gating before message aggregation, and then updates the resulting node and edge states. This updates Eqn.~\eqref{eq:layer_loop} into a recurrent node--edge--global loop:
\begin{equation}
\mathbf{x}_{i}^{(\ell)}
\rightarrow
\mathbf{e}_{ij}^{(\ell)}
\rightarrow
\mathbf{y}^{(\ell)}
\rightarrow
g_{ij}^{(\ell)}, s_i^{(\ell)}
\rightarrow
\mathbf{x}_{i}^{(\ell+1)}
\rightarrow
\mathbf{e}_{ij}^{(\ell+1)}
\rightarrow
\mathbf{y}^{(\ell+1)}.
\label{eq:node_edge_global_loop}
\end{equation}
This coupling distinguishes GenGNN from standard message passing for denoising, as shown in Eqn.~\eqref{eq:layer_loop}. Instead, Eqn.~\eqref{eq:node_edge_global_loop} utilizes edge and global channels to actively shape local aggregation, while refreshing local states.

GenGNN is capable of instantiating standard aggregation operators, including GCN- or GIN-style node updates~\cite{kipf2016semi,xu2018powerful}. However, GenGNN's design is centered around testing whether competitive graph generation for message-passing is possible when node updates are embedded inside a coupled edge-persistent and graph-conditioned denoising loop.
This lets us test whether the generative gap between message passing and more expressive graph generation backbones is due to theoretical expressivity, or how node, edge, and graph-level signals are correlated during denoising.

\subsection{Effective Message-Passing Denoisers: A Theoretical Perspective}
\label{ssec:theory_residual_denoising}

Discrete graph generation requires a denoising backbone to repeatedly update noisy node, edge, and graph states across many reverse diffusion steps. Strong graph generation models often use deep denoisers (10$+$ layers) to capture long-range dependencies \cite{qin2024defog,vignac2022digress}, sparking concern around message-passing's tendency to oversmooth gradients as layer-depth increases~\cite{li2018deeper}. In standard prediction tasks, oversmoothing describes the collapse of node representations toward low-dimensional dominant eigenspace. This issue is more severe in graph generation: if a denoiser collapses during reverse diffusion, then the model may lose node-wise distinctions for recovering valid graph structure, leading to degenerate samples, as observed in Figure~\ref{fig:degen_planar}.

Graph generation's key distinction from standard GNN prediction is how diffusion denoisers do not process a single clean graph. Instead, the denoiser receives a sequence of corrupted graphs, $\{G_t\}_{t=1}^{T}$, where late noising steps often make categorical node states uninformative. Residual connections therefore serve as an \emph{identity anchor}: preserving timestep-invariant structural information, such as positional encodings, while the message-passing operator learns a timestep-dependent denoising correction.

\paragraph{Setup.}
We follow the generalized oversmoothing perspective of \citet{scholkemper2024residual}. Let $v\in\mathbb{R}^{n}$ be a unit vector and define the orthogonal projector $P_v := I-vv^\top$. For node representations $X\in\mathbb{R}^{n\times d}$, define the distance from collapse to $\mathrm{span}\{v\}$ as
\begin{equation}
\label{eq:main_mu_v}
    \mu_v(X) := \|P_v X\|_F^2 .
\end{equation}
A small $\mu_v(X)$ indicates node representations nearly align with a dominant one-dimensional subspace, while larger values indicate node-wise information remains dispersed. This measure is useful because several common oversmoothing quantities, including: distances to the dominant eigenspace, Dirichlet-energy-based measures, and constant-subspace distances, are equivalent up to constants under standard assumptions \cite{cai2020note,oono2019graph,rusch2023survey,wu2023demystifying,scholkemper2024residual}.

At reverse diffusion step $t$, let the denoiser input be
\begin{equation}
\label{eq:main_denoiser_input}
    X_{\mathrm{in}}^{(t)}
    =
    \big[\,X_t \;\; X_{\mathrm{ENC}}\,\big],
\end{equation}
where $X_t$ is the noisy node state and $X_{\mathrm{ENC}}$ denotes structural encodings supplied to the denoiser. We write one residual message-passing block abstractly as
\begin{equation}
\label{eq:main_residual_denoiser}
    X_{\mathrm{out}}^{(t)}
    =
    f_\theta\!\left(
        X_{\mathrm{in}}^{(t)}, E_{\mathrm{in}}^{(t)}, y_{\mathrm{in}}^{(t)}, \mathrm{mask}
    \right)
    +
    S_X X_{\mathrm{in}}^{(t)} ,
\end{equation}
where $G_\theta(\cdot)$ is the learned message-passing denoising update and $S_X X_{\mathrm{in}}^{(t)}$ is the residual identity channel, possibly followed by a learned linear projection or normalization. This decomposition separates the denoiser into a \emph{correction path}, $G_\theta(\cdot)$, and an \emph{identity path}, $S_X X_{\mathrm{in}}^{(t)}$. The identity path is useful only if it carries information that does not collapse under projection. We state this as a structural non-degeneracy condition.

\begin{table*}[h!]
\centering
\small
\setlength{\tabcolsep}{4.6pt}
\renewcommand{\arraystretch}{1.08}
\begin{tabular}{l|c|cc|cc|cc|c|c}
\cmidrule(lr){1-10}
& \multicolumn{1}{c|}{Comm20}
& \multicolumn{2}{c|}{Tree}
& \multicolumn{2}{c|}{Planar}
& \multicolumn{2}{c|}{SBM}
& \multicolumn{1}{c|}{QM9}
& \multicolumn{1}{c}{ZINC250K} \\
\cmidrule(lr){2-2} \cmidrule(lr){3-4} \cmidrule(lr){5-6} \cmidrule(lr){7-8} \cmidrule(lr){9-9} \cmidrule(lr){10-10}
& Avg. $\downarrow$
& Avg. $\downarrow$ & V.U.N. $\uparrow$
& Avg. $\downarrow$ & V.U.N. $\uparrow$
& Avg. $\downarrow$ & V.U.N. $\uparrow$
& Val. $\uparrow$ & Val. $\uparrow$  \\
\cmidrule(lr){1-10} 

\multicolumn{10}{l}{\textbf{DeFoG}} \\
\cmidrule(lr){1-10}
PPGN
& 2.2 $\pm$ 0.86
& 1.7 $\pm$ 0.44 & \textbf{93 $\pm$ 4.6}
& 2.3 $\pm$ 0.63 & 93 $\pm$ 3.4
& OOM & OOM
& 98.69  $\pm$  0.86
& \textbf{95.59 $\pm$  0.02} \\

GT
& 1.9 $\pm$ 0.66
& 2.6 $\pm$ 0.8 & 91 $\pm$ 2.0
& 1.9 $\pm$ 0.64 & 93 $\pm$ 3.4
& 2.7 $\pm$ 0.57 & 72 $\pm$ 2.9
& 99.25 $\pm$ 0.96
& 94.25 $\pm$  0.4 \\

GenGNN
& \textbf{1.8 $\pm$ 0.61}
& \textbf{1.4 $\pm$ 0.29} & 91 $\pm$ 2.6
& \textbf{1.5 $\pm$ 0.24} & 93 $\pm$ 3.9
& \textbf{1.9 $\pm$ 0.37} & 73 $\pm$ 4.1
& \textbf{99.26 $\pm$ 0.0}
& 95.55 $\pm$ 0.79 \\



\cmidrule(lr){1-10}
\textit{$\Delta$\% PPGN}
& \colorup{+19\%} & \colorup{+18\%} & \colordown{-2\%} & \colorup{+35\%} & 0\% & -- & --
& \colorup{+0.6\%} & 0\% \\
\cmidrule(lr){1-10}
\textit{$\Delta$\% GT}
& \colorup{+5\%} & \colorup{+86\%} & 0\% & \colorup{+21\%} & 0\% & \colorup{+29\%} & \colorup{+1\%}
& 0\% & \colorup{+1.4\%} \\

\cmidrule(lr){1-10}
\multicolumn{10}{l}{\textbf{DiGress}} \\
\cmidrule(lr){1-10}
PPGN
& 1.9 $\pm$ 0.77
& 1.9 $\pm$ 0.9 & \textbf{92 $\pm$ 2.4}
& 1.9 $\pm$ 0.63 & \textbf{89 $\pm$ 4.7}
& OOM & OOM
& 99.23 ± 0.07
& OOM \\

GT
& 1.9 $\pm$ 0.62
& 1.8 $\pm$ 0.92 & 90 $\pm$ 2.9
& \textbf{1.2 $\pm$ 0.25} & 85 $\pm$ 5.3
& 1.9 $\pm$ 0.37 & 75 $\pm$ 4.9
& 99.36 $\pm$ 1.2
& 94.19 $\pm$ 1.81 \\

GenGNN
& \textbf{1.7 $\pm$ 0.13}
& \textbf{1.4 $\pm$ 0.34} & 91 $\pm$ 3.5
& 1.5 $\pm$ 0.5 & 85 $\pm$ 7.5
& 1.9 $\pm$ 0.7 & \textbf{77.2 $\pm$ 3.7}
& \textbf{99.49 $\pm$ 0.36}
& \textbf{96.24 $\pm$ 0.15} \\



\cmidrule(lr){1-10}
\textit{$\Delta$\% PPGN}
& \colorup{+10\%} & \colorup{+23\%} & \colordown{-1\%} & \colorup{+22\%} & \colordown{-4\%}  & -- & --
& \colorup{+0.2\%} & -- \\
\cmidrule(lr){1-10}
\textit{$\Delta$\% GT}
& \colorup{+10\%} & \colorup{+27\%} & \colorup{+1\%} & \colorup{+25\%} & \colorup{+0.1\%} & 0\% & \colorup{+2\%}
& 0\% & \colorup{+2\%} \\

\cmidrule(lr){1-10}
\end{tabular}
\caption{DeFoG and DiGress results with GenGNN, PPGN, and GT bacbones on synthetic datasets (Comm20, Tree, Planar, SBM) and molecular datasets (QM9, ZINC250K).
Lower is better for Avg Ratio; Higher is better for V.U.N. and Validity (Val.). For readability, we select the best-performing GenGNN baseline.
Results are mean $\pm$ std.\ across five sample runs. We color improvements in performance \colorup{blue} and decreases \colordown{red}.}
\label{tab:synth_mole_results}
\end{table*}

\begin{assumption}\label{ass:main_residual_nondeg}
There exists $\gamma>0$ such that, for each unit vector $v$ associated with a potential oversmoothing direction,
\begin{equation}
\label{eq:main_residual_gamma}
    \mu_v\!\left(S_X[0\;\;X_{\mathrm{ENC}}]\right)
    =
    \left\|P_v S_X[0\;\;X_{\mathrm{ENC}}]\right\|_F^2
    \ge \gamma .
\end{equation}
\end{assumption}

This condition states that the residual channel preserves nontrivial node-wise variation from the structural encoding alone, even when the noisy categorical channel $X_t$ is removed.
Thus, when the forward process drives $X_t$ toward an uninformative marginal, the denoiser still receives an anchored structural signal through the residual path. 
This assumption is consistent with the role of residual connections in deep architectures \cite{he2016deep,peebles2023scalable} and with recent theory showing that residual connections and normalization can prevent complete GNN oversmoothing \cite{scholkemper2024residual}. We additionally assume that the learned correction path has bounded orthogonal energy.

\begin{assumption}\label{ass:main_backbone_bound}
There exists $C<\infty$ such that, for every reverse diffusion step $t$,
\begin{equation}
\label{eq:main_backbone_bound}
    \left\|
    P_v G_\theta\!\left(
        X_{\mathrm{in}}^{(t)}, E_{\mathrm{in}}^{(t)}, y_{\mathrm{in}}^{(t)}, \mathrm{mask}
    \right)
    \right\|_F^2
    \le C .
\end{equation}
\end{assumption}

Assumption~\ref{ass:main_backbone_bound} does not require the message-passing update to be small; it only requires the orthogonal component of a learned correction remains bounded. In practice, this aligns with the intuition that denoising backbones should learn controlled corrections between consecutive reverse steps rather than overwrite node-wise information at every layer. We treat this as Theorem~\ref{thm:main_uniform_noncollapse}.

\begin{theorem}
\label{thm:main_uniform_noncollapse}
Suppose the denoising block satisfies Eq.~\eqref{eq:main_residual_denoiser}
and Assumptions~\ref{ass:main_residual_nondeg}--\ref{ass:main_backbone_bound}.
If the residual path preserves the structural anchor at each timestep, i.e.,
\begin{equation}
\label{eq:main_timestep_anchor}
    \left\|P_vS_XX_{\mathrm{in}}^{(t)}\right\|_F^2 \ge \gamma
    \qquad \text{for all } t ,
\end{equation}
then, for every reverse diffusion step $t$,
\begin{equation}
\label{eq:main_final_uniform_noncollapse}
    \boxed{
    \mu_v\!\left(X_{\mathrm{out}}^{(t)}\right)
    \ge
    \frac{1}{2}\gamma - C
    } .
\end{equation}
In particular, if $\gamma>2C$, then
$\mu_v(X_{\mathrm{out}}^{(t)})>0$ for all $t$, and the denoiser output cannot
collapse completely to the oversmoothed subspace associated with $v$ at any
reverse diffusion step.
\end{theorem}

\begin{proof}
Fix a reverse diffusion step $t$ and a potential oversmoothing direction $v$.
Write the residual denoising block in Eq.~\eqref{eq:main_residual_denoiser} as
\[
\label{eq:main_residual_A_B}
    X_{\mathrm{out}}^{(t)} = A_t + B_t ,
\]
where
\[
\label{eq:main_A_definition}
    A_t =
    G_\theta\!\left(
        X_{\mathrm{in}}^{(t)},
        E_{\mathrm{in}}^{(t)},
        y_{\mathrm{in}}^{(t)},
        \mathrm{mask}
    \right)
\]
is the learned correction path and
\[
\label{eq:main_B_definition}
    B_t = S_X X_{\mathrm{in}}^{(t)}
\]
is the residual identity path. Applying the projector $P_v$ gives
\[
\label{eq:main_projected_decomposition}
    P_vX_{\mathrm{out}}^{(t)} = P_vA_t + P_vB_t .
\]
Let $a_t=P_vA_t$ and $b_t=P_vB_t$. By the triangle inequality,
\[
\label{eq:main_triangle_step}
    \|b_t\|_F
    =
    \|(a_t+b_t)-a_t\|_F
    \le
    \|a_t+b_t\|_F+\|a_t\|_F .
\]
Using $(r+s)^2\le 2r^2+2s^2$ for $r,s\ge0$ gives
\[
\label{eq:main_scalar_step}
    \|b_t\|_F^2
    \le
    2\|a_t+b_t\|_F^2
    +
    2\|a_t\|_F^2 .
\]
Rearranging yields
\[
\label{eq:main_inline_two_term_bound}
    \|P_vX_{\mathrm{out}}^{(t)}\|_F^2
    =
    \|a_t+b_t\|_F^2
    \ge
    \frac{1}{2}\|b_t\|_F^2
    -
    \|a_t\|_F^2 .
\]
Substituting the definitions of $a_t$ and $b_t$ gives
\begin{align}
\label{eq:main_explicit_residual_lower_bound}
    \mu_v\!\left(X_{\mathrm{out}}^{(t)}\right)
    &=
    \left\|P_vX_{\mathrm{out}}^{(t)}\right\|_F^2
    \nonumber\\
    &\ge
    \frac{1}{2}
    \left\|P_vS_XX_{\mathrm{in}}^{(t)}\right\|_F^2
    -
    \left\|
    P_vG_\theta\!\left(
        X_{\mathrm{in}}^{(t)},
        E_{\mathrm{in}}^{(t)},
        y_{\mathrm{in}}^{(t)},
        \mathrm{mask}
    \right)
    \right\|_F^2 .
\end{align}
By the timestep anchor condition in Eq.~\eqref{eq:main_timestep_anchor},
the first term is lower bounded by $\gamma/2$. By
Assumption~\ref{ass:main_backbone_bound}, the second term is upper bounded
by $C$. Therefore,
\begin{equation}
\label{eq:main_final_uniform_noncollapse}
    \mu_v\!\left(X_{\mathrm{out}}^{(t)}\right)
    \ge
    \frac{1}{2}\gamma - C
\end{equation}
Since $t$ is arbitrary, the bound holds uniformly across all reverse diffusion steps.
When $\gamma>2C$, the right-hand side is positive, so
$\mu_v(X_{\mathrm{out}}^{(t)})>0$ for all $t$.
\end{proof}

{\bf Takeaway}: Theorem~\ref{thm:main_uniform_noncollapse} explains why residual connections are more than an optimization convenience for message-passing graph denoisers. During late forward noising, $X_t$ may approach a stationary categorical marginal and become weakly informative. A purely local message-passing denoiser must then infer structure primarily from noisy node and edge states, making it vulnerable to oversmoothing. In contrast, a residual-anchored denoiser preserves a timestep-invariant structural signal through $S_XX_{\mathrm{in}}^{(t)}$, while the learned message-passing path contributes a bounded denoising correction. Thus, the identity path grounds informative node-wise variation across reverse diffusion steps.

\begin{corollary}
\label{cor:main_anchored_robustness}
If the forward process drives $X_t$ toward an uninformative stationary marginal as $t\rightarrow T$, but the residual path satisfies Eq.~\eqref{eq:main_timestep_anchor}, then the reverse denoiser retains a non-vanishing component orthogonal to the oversmoothing direction at every timestep whenever $\gamma>2C$.
\end{corollary}

{\bf Takeaway}: Corollary~\ref{cor:main_anchored_robustness} connects current understanding on GNN oversmoothing to generative denoising~\cite{li2018deeper,cai2020note,oono2019graph,rusch2023survey,wu2023demystifying, scholkemper2024residual}. Residual channels preserve stable structural information, enabling the message-passing loop to learn node-, edge-, and graph-wise corrections for denoising input samples. All while GenGNN solely relies on local message-passing for model optimization. This provides a theoretical explanation for the empirical behavior in Section~\ref{ssec:metric_scaling}: residuals keep message passing effective at layer depths exceeding 12.

%% file: sections/results.tex
\section{Experiments}

We empirically demonstrate the validity of our anchored-denoiser claim through strong performance of GenGNN as a new backbone for DiGress and DeFoG across standard graph generation benchmark datasets. We conduct our experimentation in multiple phases:

\subsection{Results}

\begin{figure}[h!]
  
  \begin{center}
    \centerline{
    \includegraphics[width=0.9\columnwidth]{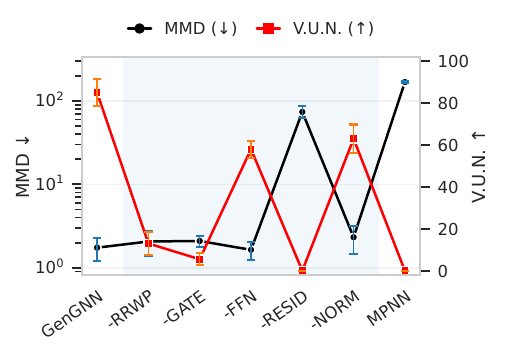}
    }
    \vskip -2em
    \caption{
      The change in MMD and V.U.N. across individual ablated components of the GenGNN framework (log-scaled), with a simple GNN backbone (on right). 
      } 
      \label{fig:tree_ablate_diag}
   \end{center}
    
\end{figure}

\paragraph{Experimental Settings} We test GenGNN's ability to generate diverse graph structural and molecular  characteristics with modular variants of GNN~\cite{gilmer2017neural}, GCN~\cite{kipf2016semi}, GINe~\cite{xu2018powerful} versus the LocalPPGN~\cite{bergmeister2023efficient} and GT~\cite{vignac2022digress} backbones across the: Comm20, Planar, SBM~\cite{martinkus2022spectre}, Tree~\cite{bergmeister2023efficient}, QM9~\cite{wu2018moleculenet}, and ZINC250K~\cite{sterling2015zinc} datasets. Model checkpoints on synthetic datasets were selected based on best-validation MMD, whereas molecular datasets were selected by best Validity. We apply the same datasets and testing standards as~\cite{qin2024defog}. To set a baseline, all model hyperparameters were tuned to identical settings. 

\paragraph{Structural and Molecular Generation}
As shown in Table~\ref{tab:synth_mole_results}, GenGNN variants consistently improve or perform on-par with PPGN and GT across average MMD ratio, with relative gains for Tree, Planar, and SBM reaching 86\%, 35\%, and 29\% respectively. Notably, GenGNN reaches structural validity (V.U.N.) levels on-par with PPGN and GT, reaching 93\% V.U.N. on Planar when applied within DeFoG. This is further demonstrated by the remaining results which are \textit{always within margins of GT and PPGN, both of which constitute more powerful architectures}.

\paragraph{Large Molecule and Conditional Generation} We verify GenGNN's effectiveness on the practical applications of larger-scale molecular design and conditional pathological diagnosis. GenGNN is benchmarked versus the original GT backbone on the: MOSES~\cite{polykovskiy2020molecular}, Guacamol~\cite{brown2019guacamol}, QM9(H)~\cite{wu2018moleculenet}, and Tertiary Lympoid Structure (TLS)~\cite{madeira2024generative} datasets. GNN and GT were grid-search tuned across identical hyperparameter configurations and selected for the best-validation Validity score. 

GenGNN's use of pooling allows conditional generation under classifier-free guidance~\cite{ho2022classifier}. TLS results within Table~\ref{tab:defog_molecular} demonstrate GenGNN's ability to handle conditional generation. As shown with a V.U.N score of 93.5\%, GenGNN's local message-passing disambiguates high TLS content labels from low TLS content labels, all while generating valid planar structure. Furthermore, GenGNN achieves 91.9\% Validity on MOSES and GuacaMol. This is notable given that GenGNN only requires 72 hours to train on MOSES or GuacaMol. Our tests with the original GT backbone required nearly a week to fully-converge on either dataset~\cite{qin2024defog}. This training speed-up signals GenGNN as a practical alternative with fast training times to reach reliable performance. 

\begin{table}[h!]
\small
\setlength{\tabcolsep}{4.5pt}

\begin{tabular}{c|cc}
\toprule
\multirow{2}{*}{\textbf{Ablations}} &
\multicolumn{2}{c}{\textbf{GenGNN}} \\
\cmidrule(lr){2-3}
& P(ERank) & Val. Ratio \\
\midrule
\multicolumn{2}{c}{\textbf{QM9}} \\
\midrule
-RRWP  & 0.37674 & 1.025 \\
-RRWP,NG,GG & -0.34339 & 1.0431\\
-RRWP,NG,Res.  & 0.99974 & 0.03889  \\
\bottomrule
\end{tabular}
\caption{Pearson correlation between QM9 Validity and ERank across layers depths: 1-24, with associated Validity Ratio. All metrics are calculated for \textbf{GenGNN} within DeFoG. NG = Node Gating, GG = Global Gating, RRWP = Relative-Random Walk Positional Encodings, Res. = Residual Connections.}
\label{tab:oversmoothing_correlation}
\end{table}

\begin{table*}[h]
\centering
\small
\setlength{\tabcolsep}{5.5pt}
\renewcommand{\arraystretch}{1.1}
\begin{tabular}{l|c|ccc|ccc|c}
\toprule
\multirow{2}{*}{\textbf{Backbone}} 
& \multicolumn{1}{c|}{TLS} 
& \multicolumn{3}{c|}{MOSES} 
& \multicolumn{3}{c|}{GuacaMol}
& \multicolumn{1}{c}{QM9(H)}  \\
\cmidrule(lr){2-2} 
\cmidrule(lr){3-5}
\cmidrule(lr){6-8}
\cmidrule(lr){9-9}
& V.U.N $\uparrow$  
& Validity $\uparrow$ 
& Uniqueness $\uparrow$ 
& Novelty $\uparrow$
& Validity $\uparrow$ 
& Uniqueness $\uparrow$ 
& Novelty $\uparrow$ 
& Validity $\uparrow$ \\
\midrule
GT 
& $93.75 \pm 2.02$ 
& $91.82 \pm 0.7$ 
& $99.53 \pm 0.35$ 
& $91.95 \pm 2.38$ 
& $91.82 \pm 0.45$ 
& $99.91 \pm 0.02$ 
& $98.89 \pm 0.01$ 
& $96.8 \pm 0.001$\\
GenGNN 
& $93.61 \pm 3.53$  
& $91.94 \pm 0.52$ 
& $99.79 \pm 0.19$ 
& $92.06 \pm 2.94$ 
& $92.66 \pm 0.52$ 
& $99.92 \pm 0.02$ 
& $99.01 \pm 0.03$ 
& $96.8 \pm 0.002$\\
\midrule 
$\Delta$\% & \colordown{-0.1\%} & \colorup{+0.1\%} & \colorup{+0.2\%} & \colorup{+0.1\%} & \colorup{+0.8\%} & 0\% & 0\% & 0\% \\
\bottomrule
\end{tabular}
\caption{DeFoG results for conditional generation with TLS and large-scale molecular generation results for MOSES, GuacaMol, and QM9(H) benchmarks. All results are mean $\pm$ std.\ across five sample runs. We color improvements in performance \colorup{blue} and decreases \colordown{red}.}
\label{tab:defog_molecular}
\end{table*}

\begin{figure*}[h!]
  \begin{center}
    \centerline{\includegraphics[width=0.98\textwidth]{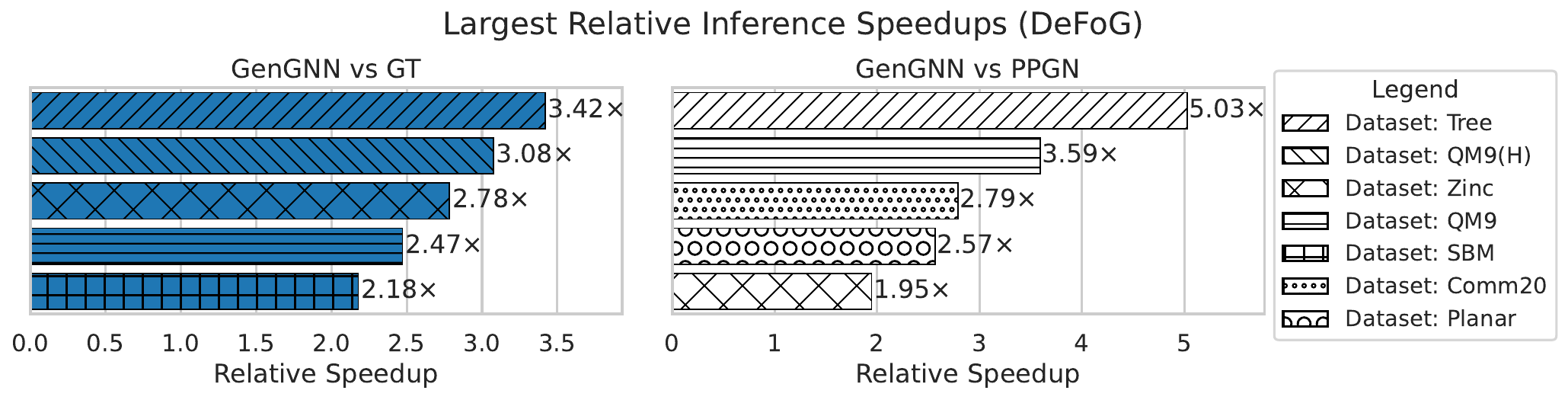}}
    \caption{
      The top-5 relative inference speedups for GenGNN vs. PPGN and GT denoising backbones across permutations of tested datasets. Individual colors represent GT (blue) and PPGN (white), hatching corresponds to a given dataset. 
    }
    \label{fig:inf_time}
  \end{center}
\end{figure*}

\subsection{Ablation Study}

To isolate the importance of each GenGNN component, we ablate individual parts of the GenGNN architecture. Figure~\ref{fig:tree_ablate_diag} illustrates results for the Tree dataset, where alternating plots for MMD and V.U.N indicate the severity of impact on downstream generation induced by removing a single component. 



Our analysis of Figure~\ref{fig:tree_ablate_diag} draws three separate conclusions. {\bf (1)} The most drastic change occurs to V.U.N, where removing any GenGNN component causes an $~$20\% decrease; sometimes dropping below 5\% overall. {\bf (2)} A near-inverse correlation occurs between MMD and V.U.N, as shown when GenGNN drops to 0\% V.U.N and 77 MMD after ablating per-layer residuals. This is the most severe decrease, demonstrating a need for residual in message-passing denoisers and providing empirical support for Theorem~\ref{thm:main_uniform_noncollapse}. {\bf (3)} The removal of RRWP and Gating cause V.U.N to drop near 10\%. However MMD ratios remains within margins of the fully-enabled GenGNN. This aligns with how powerful expressivity is required to capture difficult substructures, such as branch and leaf-nodes in tree graphs~\cite{xu2018powerful, feng2022powerful}.

\subsection{Oversmoothing Analysis}

\citet{zhang2025rethinking} determined that ERank~\cite{roy2007effective} is sensitive to node-signal oversmoothing within message-passing layers. Therefore, measuring correlation between QM9 Validity and ERank across deeper GenGNN layers makes it possible to further validate Theorem~\ref{thm:main_uniform_noncollapse}.

Results from Table~\ref{tab:oversmoothing_correlation} are calculated across layer depths $1-24$, increasing by increments of $4$. There are two strong results. The first occurs when per-layer residual connections are ablated within GenGNN, leading to near-perfect correlation (\textasciitilde$0.99$) between degraded performance and node-embedding rank, demonstrating clear oversmoothing. The second indicates \textit{residuals make GenGNN resilient to oversmoothing}. This is measured by negative/weak rank correlation and validity ratio increasing with more layers (>$1.0$), \textit{indicating measurable performance improvement for a deeper GenGNN}.

\begin{figure*}[h!]
  \begin{center}
    \centerline{\includegraphics[width=0.95\textwidth]{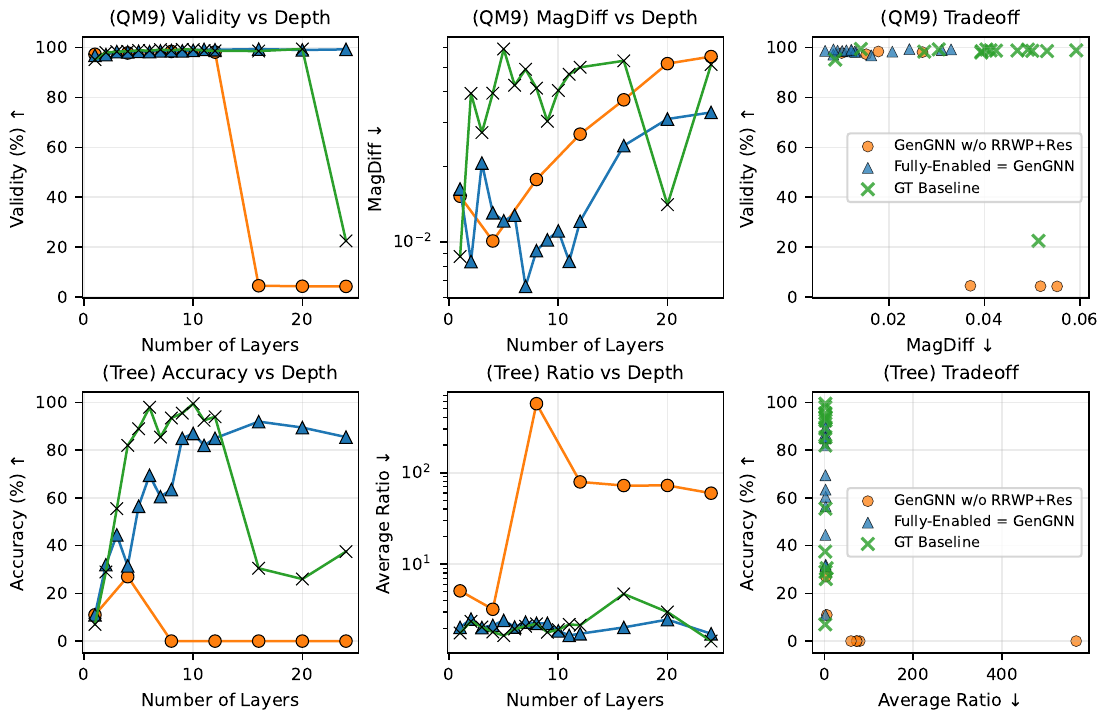}}
    \caption{
      The change in Validity (top-left), Accuracy (bottom-left),  MagDiff (top-center), and Accuracy (bottom-center) from layer depths 1 to 24 for the fully-enabled GenGNN and GT frameworks vs. GenGNN with RRWP and residual-normalization ablated, averaged over five runs. (top-right) The trade-off between MagDiff and Validity for the QM9 Dataset. (bottom-right) The trade-off between Average MMD Ratio and Accuracy for the Tree Dataset.
    }
    \label{fig:qm9_ablate}
  \end{center}
\end{figure*}

\subsection{Inference Time Comparison}

Figure~\ref{fig:inf_time} shows the difference in inference time for GenGNN relative to PPGN and GT across 6 datasets. This is calculated by the comparative decrease in number of seconds to complete sampling of five folds between the two encoders. {\bf We observe that GenGNN can achieve a speedup of 2.5-5x}. Furthermore, given that GenGNN often achieves comparable or better performance than existing encoders, the efficiency boost does not come at the expense of downstream performance.

\subsection{Metric-Space Scaling Investigation} \label{ssec:metric_scaling}

In order to isolate whether GenGNN learns latent representations which effectively capture the necessary modes for proper graph generation, we integrate a GINe-variant of GenGNN into the randomly-initialized feature extractor framework introduced by~\cite{thompson2022evaluation}. The feature extractor is given reference and predicted graphs as input. From which, the extractor produces embeddings to evaluate latent diversity of learned model representations. The MagDiff metric, as introduced in Equation~\ref{eq:magdiff}, then measures difference between both embeddings. MagDiff provides the practical benefit of near-perfect correlation for identifying structural diversity, even while measuring highly-noisy input samples~\cite{limbeck2024metric}. 

In Figure~\ref{fig:qm9_ablate}, both rows of left and center subplots gradually increase the number of layers for: 1) GT, 2) a fully-enabled GenGNN baseline, 3) and GenGNN with RRWP and Residuals ablated. The top-right subplot indicates the tradeoff between MagDiff and Validity for each diffusion backbone. The bottom-right subplot is restricted to Average MMD Ratio and Accuracy. This is due to the Tree dataset containing non-unique node embeddings, a restriction from a fixed node size of 64. 

The scaling analysis is conducted in this way to determine is there is visible correlation for latent diversity of generated samples (top-right subplot) to Validity and Accuracy. Both left-most subplots indicate GenGNN is resilient to oversmoothing. As demonstrated in two ways: 1) GenGNN's gradual performance increase with the number of layers, 2) the stability of GenGNN's performance at higher layer counts relative to GT or when residuals are ablated. The bottom center and rightmost subplots demonstrate GenGNN consistently performs on-par with GT on the Tree dataset, regardless of layer depth. The top center and rightmost subplots demonstrate GenGNN often clusters around MagDiff values lower than GT, indicating local message-passing can capture more latent diversity in graph substructure for the QM9 dataset. Given these clear trends, GenGNN overcomes oversmoothing in a way where it's \textit{learned representation translates directly to improved downstream performance}; just as might be expected from a GT backbone. This indicates clear practical benefit for GenGNN as a framework to uplift local message-passing for graph generation.



%% file: sections/conclusion.tex
\section{Conclusion}
This paper takes aim at graph transformer backbones typically used in discrete graph generation. In practice, existing graph generative models tend to use powerful encoders, which come with the downside of being far less efficient than vanilla GNNs. We propose GenGNN, a simple and modular GNN framework incorporating best practice to demonstrate comparable performance to expressive systems with improved efficiency. Across a variety of real-world benchmarks, GenGNN achieves similar if not better performance relative to GT or PPGN with 2-5x inference speedup. Further analysis shows GenGNN is capable of capturing substructures not previously possible with vanilla message-passing. We believe this opens the opportunity for faster and more performant graph generation models. Future work could center around applications of local message-passing in graph generation, such as implementations for differential privacy or algorithmic reasoning.

